\documentclass{article}

\PassOptionsToPackage{numbers, compress}{natbib}
\usepackage[preprint]{neurips_2026}

\usepackage[utf8]{inputenc} 
\usepackage[T1]{fontenc}    
\usepackage{hyperref}       
\usepackage{url}            
\usepackage{booktabs}       
\usepackage{amsfonts}       
\usepackage{nicefrac}       
\usepackage{microtype}      
\usepackage{xcolor}         

\usepackage[ruled,vlined,linesnumbered]{algorithm2e}
\usepackage{amsmath,amssymb,bm}
\usepackage{multirow}
\usepackage{cancel}
\usepackage{comment}
\usepackage{array}
\newcolumntype{C}[1]{>{\centering\arraybackslash}m{#1}}
\usepackage{makecell}
\usepackage[normalem]{ulem}
\usepackage{graphicx}
\usepackage{caption}
\usepackage{wrapfig}
\usepackage{hhline}
\usepackage[dvipsnames]{xcolor}
\usepackage{stmaryrd}

\usepackage{amsthm}

\newtheorem{proposition}{Proposition}
\newtheorem{lemma}{Lemma}

\newtheorem{definition}{Definition}
\newtheorem{assumption}{Assumption}

\DeclareMathOperator*{\argmin}{arg\,min}

\title{Sentence Curve Language Models}

\author{%
  DongNyeong Heo \\
  Ulsan National Institute of Science and Technology \\
  Ulsan, Republic of Korea 44919 \\
  \texttt{sjglsks@gmail.com} \\
  \And
  Taehwan Kim \\
  Ulsan National Institute of Science and Technology \\
  Ulsan, Republic of Korea  44919 \\
  \texttt{taehwankim@unist.ac.kr} \\
  \And
  Heeyoul Choi \\
  Handong Global University \\
  Pohang, Republic of Korea 37554 \\
  \texttt{hchoi@handong.edu} \\
}

\begin{document}

\maketitle

\begin{abstract}
  Language models (LMs) are a central component of modern AI systems, and diffusion language models (DLMs) have recently emerged as a competitive alternative. Both paradigms rely on word embeddings not only to represent the input sentence, but also—implicitly—to represent the target sentence that backbone models are trained to predict. We argue that such static embedding of the target word is insensitive to neighboring words, encouraging locally accurate word prediction while global sentence structure is less emphasized. To address this, we propose a continuous sentence representation, termed \textbf{sentence curve}, defined as a spline curve whose control points affect multiple words in the sentence. Based on this representation, we introduce \textbf{sentence curve language model} (SCLM), which extends DLMs to predict sentence curves instead of the static word embeddings. We theoretically show that sentence curve prediction induces a regularization effect that promotes global structure modeling, and characterize how different sentence curve types affect this behavior. Empirically, SCLM achieves state-of-the-art performance among DLMs on IWSLT14 and WMT14, shows stable training without burdensome knowledge distillation, and demonstrates promising potential compared to discrete DLMs on LM1B.
\end{abstract}




\section{Introduction}
Language models (LMs) have demonstrated remarkable impact in real-world applications \cite{team2023gemini, achiam2023gpt}. Recently, diffusion language models (DLMs) \cite{sahoo2024simple, nie2025large} have attracted increasing attention due to their distinctive advantages, such as iterative refinement, bidirectional context awareness, and fine-grained controllability. A fundamental component of both LMs and DLMs is the word embedding \cite{mikolov2013distributed}, which provides a distributed representation of individual words. Accordingly, a sentence is represented as a sequence of word embeddings. In particular, not only the input sentence but also the target sentence, which the model is enforced to predict, can be viewed as a sequence of word embeddings; we provide a detailed justification of this perspective in Section~\ref{sec:word_embedding_prediction}.

Although word embeddings can encode rich lexical information, they remain static with respect to neighboring words within a sentence. Using such embeddings as target representations therefore forces the model to predict a sequence of context-independent vectors. We argue that this static target representation can cause models to overfit to word-level local structure while neglecting sentence-level global structure. This issue becomes more pronounced when input representations are noisy, as in non-autoregressive (non-AR) decoding and DLMs, where it can lead to practical problems such as the multimodality problem—i.e., weak modeling of dependencies between target tokens \cite{gu2017non, feng2025theoretical}—sometimes referred to as `mode mixing' or `mode averaging'.
Despite extensive research on input representations, the representation of target sentences has received comparatively little attention.

\begin{center}
    \includegraphics[width=1.0\linewidth]{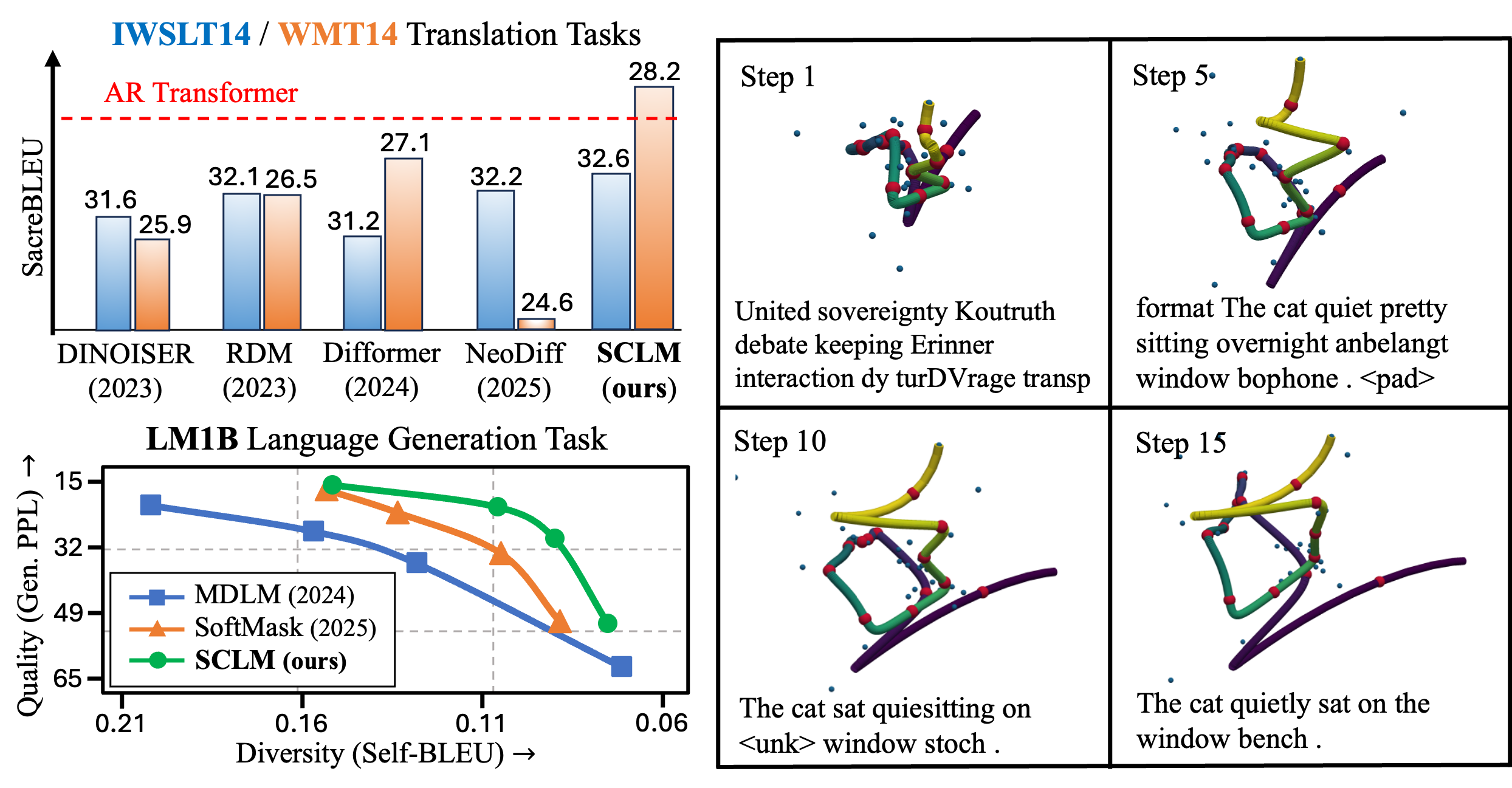}
    \captionof{figure}{\textbf{Left}: Performance comparison of SCLM with prior models on translation tasks (IWSLT14 and WMT14 English–German, top) and generic language modeling (LM1B, bottom). \textbf{Right}: Example of sentence curve generation by SCLM in a translation task, where the reference sentence is \textit{“The cat sat quietly on the window sill.”} Red points on the curve represent the final word embeddings.
    } 
    \label{fig:front_page}
\end{center}

In time-series modeling, \cite{hug2020introducing} proposed representing discrete data using continuous spline curves and training neural networks to predict the curve rather than individual data points, effectively mitigating the multimodality problem across several time-series domains. Motivated by the work and the fact that a sentence is likely a sequence of symbolic elements, we propose \textit{\textbf{sentence curve}}, a spline curve that traverses the word embeddings in vector space. Because the shape of sentence curve changes coherently with constituent words in the sentence, it can capture sentence-level global structure. Building on non-AR and semi-AR decoding models, such as DLMs, we introduce \textit{\textbf{sentence curve language model}} (SCLM), which predicts sentence curves and generates sentences through them, rather than relying on conventional static word embeddings (see the right side of Figure~\ref{fig:front_page}, and Figure~\ref{fig:extra_curve_generation_examples}). We further provide theoretical evidence that sentence curve prediction encourages models to focus more on global structure compared to the conventional approach. 

Through extensive experiments, we demonstrate that our proposed SCLMs achieve state-of-the-art (SOTA) performance among Gaussian DLMs on machine translation benchmarks, including IWSLT14 \cite{cettolo2014report} and WMT14 \cite{bojar2014findings} (Figure~\ref{fig:front_page}, top-left), and even surpass AR Transformers \cite{vaswani2017attention} on some tasks. We further show that SCLMs can be trained stably without sequence-level knowledge distillation (KD) \cite{kim2016sequence}, which is commonly used to mitigate the multimodality problem and requires a burdensome pre-trained AR teacher. Additional analyses, including token-level distance correlation and qualitative translation examples, consistently indicate that SCLM effectively addresses this issue. Beyond its advantages on translation tasks, we further evaluate SCLM on generic language modeling using LM1B, comparing it with prior discrete DLMs such as MDLM \cite{sahoo2024simple} and SoftMask \cite{hersche2025soft} (Figure~\ref{fig:front_page}, bottom-left). Together, we argue that these theoretical and empirical results suggest sentence curves as a promising direction for future non-AR or semi-AR models.

Our contributions can be summarized as follows:
\begin{itemize} 
    \item \textbf{Sentence Curve-based Target Reformulation}: We propose \textit{sentence curves}, a continuous curve-based representation of sentences, as an alternative target for language modeling. 
    \item \textbf{Sentence Curve Language Model}: We incorporate sentence curve prediction into non-AR models and instantiate it with DLMs as the sentence curve language model (SCLM). 
    \item \textbf{Theoretical Framework for Global Structure Modeling}: We develop a framework explaining how sentence curve prediction encourages sentence-level global structure regularization. 
    \item \textbf{Empirical Validation}: Through extensive experiments and analyses, we show that SCLM effectively mitigates the multimodality problem in DLMs and highlights its potential for non-AR language modeling. 
\end{itemize}

\section{Backgrounds}
\subsection{Notational Preliminaries}


Throughout this paper, we consider words and their embedding vectors \cite{mikolov2013distributed}. A word $y \in \mathcal{V}$ is a categorical variable over the vocabulary $\mathcal{V}$, with embedding $\mathbf{e}_y=E \delta_y \in \mathbb{R}^{d}$, where $\delta_y$ is a one-hot vector and $E=[\mathbf{e}_1,\cdots, \mathbf{e}_{|\mathcal{V}|}]\in \mathbb{R}^{d \times |\mathcal{V}|}$ is the embedding matrix. A sentence of length $L$ is denoted by $Y = (y_1, \cdots, y_L)$, or equivalently by its embedding sequence $E_Y=[\mathbf{e}_{y_1},\cdots,\mathbf{e}_{y_L}]$. We use $p_{\theta}(Y)$ to denote the model likelihood and $p_{data}(Y)$ the data distribution. Optimization is performed using cross-entropy (CE) and Kullback–Leibler (KL) divergence, which are formalized as follows: $CE_{Y}(\theta) = \mathbb{E}_{p_{data}(Y|X)}[-\log p_{\theta}(Y|X)]$ and $KL_{Y}(\theta) = \mathbb{E}_{p_{data}(Y|X)}[\log \frac{p_{data}(Y|X)}{p_{\theta}(Y|X)}]$ where $X$ denotes a conditioning variable (e.g., a source sentence or a noised input).



\subsection{Diffusion Language Models}
\label{subsec:diffusion_language_models}

Similar to diffusion models \cite{ho2020denoising}, DLMs consist of a forward diffusion process that progressively adds noise and a reverse denoising process that removes it. Depending on whether these processes operate in continuous or discrete spaces, DLMs are categorized into Gaussian (continuous) and discrete variants. We briefly review Gaussian DLMs and contrast them with discrete DLMs.

In Gaussian DLMs, a discrete target sentence $Y$ is first mapped into a continuous embedding sequence $E_Y=[\mathbf{e}_{y_i}]_{i=1}^{L}$, which is treated as clean latent variables $[\mathbf{e}_i^{0}]_{i=1}^{L}$. The forward and reverse processes on a latent variable $\mathbf{e}^{0}\in [\mathbf{e}_{i}^{0}]_{i=1}^{L}$ are, respectively, defined as
\begin{align}
    q(\mathbf{e}^{t}|\mathbf{e}^{0}) = \mathcal{N}\left( \mathbf{e}^{t} ; \sqrt{\bar{\alpha}^t}\mathbf{e}^{0}, (1-\bar{\alpha}^t)\mathbf{I}_d \right), 
    \qquad p_{\theta}(\mathbf{e}^{t-1}|\mathbf{e}^{t}) = \mathcal{N}\left( \mathbf{e}^{t-1}; \mathbf{\mu}_{\theta}(\mathbf{e}^{t},t), (1-\alpha^t)\mathbf{I}_d \right). \nonumber
\end{align}
Here, $\mathbf{e}^t$ denotes the latent variable at diffusion step $t\in\{1,\cdots,T\}$, $\alpha^t$ is a predefined noise schedule, and $\bar{\alpha}^t=\prod_{i=1}^{t}\alpha^i$. $\mathcal{N}(\cdot;\mathbf{\mu},\mathbf{\Sigma})$ denotes the Gaussian distribution that has $\mathbf{\mu}$ mean and $\mathbf{\Sigma}$ covariance parameters. $\mathbf{I}_d$ is $d \times d$ identity matrix. Most DLMs adopt an efficient parameterization in which the reverse model directly predicts the clean latent variable $\hat{\mathbf{e}}_{\theta}^{0}(\mathbf{e}^{t},t) := \frac{1}{\sqrt{\bar{\alpha}^t}}\mathbf{\mu}_{\theta}(\mathbf{e}^{t},t)$. Training of Gaussian DLMs is performed by minimizing the following objective:
\begin{align}
    \mathcal{L}_{g}(y;\theta,E) &= \mathbb{E}_{t,\mathbf{e}^t} \left[ \|\mathbf{e}^{0}-\hat{\mathbf{e}}_{\theta}^{0} \|^{2} \right] -\log p(y|E^\top\hat{\mathbf{e}}_{\theta}^{0}). \label{eq:gdlm_objective}
\end{align}
The first term (diffusion loss) encourages accurate reconstruction of the clean embedding, while the second term (anchor loss) aligns the predicted continuous latent variable with discrete word targets. 

In contrast, discrete DLMs operate directly in the discrete word space, treating the target sentence as clean categorical latent variables $Y^0 := Y$. Both forward and reverse processes are modeled using categorical distributions, and training minimizes a negative evidence lower bound objective (ELBO). In MDLM \cite{sahoo2024simple}, this objective simplifies to a noise-scaled CE loss:
\begin{align}
    \mathcal{L}_{m}(y;\theta,W) = -\mathbb{E}_{t,y^t}\left[\frac{\alpha^{t'}}{1-\alpha^t} \log p_{\theta}(\hat{y}^0|y^t,t) \right], \quad p_{\theta}(\hat{y}^0|y^t,t) = \text{softmax}(W^\top \mathbf{h}_{\theta}(y^t,t)), \label{eq:mdlm_projection_layer}
\end{align}
where $W\in\mathbb{R}^{d \times |\mathcal{V}|}$ denotes the projection (“logit”) matrix and $\mathbf{h}_{\theta}(y^t,t)$ is the hidden state produced by the backbone sequence model (e.g., Transformer \cite{vaswani2017attention}). We refer readers to \cite{sahoo2024simple} for further details.

In contrast to AR models, non-AR models (e.g., DLMs) assume conditional independence across tokens, $p(Y)=\prod_{i=1}^L p(y_i)$, enabling parallel generation. However, this leads to the multimodality problem \cite{gu2017non}: without inter-token dependencies, tokens may be generated from inconsistent modes, especially when multiple valid targets exist \cite{huang2022learning}. Prior approaches, including sequence-level knowledge distillation (KD) \cite{kim2016sequence}, and latent variable modeling \cite{shu2020latent, heo2023shared}, provide partial mitigation but do not fully resolve training instability and often require additional remedies \cite{wu2023ar, feng2025theoretical}.

\section{Static Word Embedding Prediction}
\label{sec:word_embedding_prediction}

In this section, we clarify our perspective on target representations in LMs and DLMs. Specifically, we view the backbone model parameterized by $\theta$ as predicting the embedding of the target word. In Gaussian DLMs, the objective in Eq.~\ref{eq:gdlm_objective} explicitly enforces the backbone output $\hat{\mathbf{e}}_{\theta}^0$ to match $\mathbf{e}^0$, corresponding to the target word embedding $\mathbf{e}_y$. A similar interpretation holds for discrete DLMs if the logit weight matrix $W$ in Eq.~\ref{eq:mdlm_projection_layer} (right) is regarded as a target embedding matrix: the cross-entropy loss encourages the backbone output $\mathbf{h}_{\theta}$ to align with the embedding of the target word in $W$. This connection is particularly clear when the input embedding and logit matrices are shared, i.e., $E = W$. For simplicity, we denote the target word embedding as $\mathbf{e}_y$ throughout the remainder of this paper.

We formalize this viewpoint by deriving a lemma that characterizes when the target word embedding is the optimal output of the backbone model $H_{\theta}=[\mathbf{h}_{\theta,1},\cdots,\mathbf{h}_{\theta,L}]$ under the following generic maximum likelihood estimation (MLE) pipeline:
\begin{definition}[Generic MLE Pipeline] \label{def:generic_mle_pipeline}
$H_{\theta} \rightarrow E^\top H_{\theta} \rightarrow \text{softmax}(E^\top H_{\theta}) \rightarrow \min_{\theta} \, CE_Y(\theta).$
\end{definition}
This pipeline corresponds to standard logistic regression-based classification: (1) computing backbone outputs $H_{\theta}$, (2) projecting them to logits via $E^\top H_{\theta}$, (3) applying softmax to obtain likelihoods, and (4) optimizing with CE loss. It is ubiquitous in neural classifiers, including LMs and discrete DLMs. Under mild assumptions—unit-norm constraints on $\mathbf{e}_y$ and $\mathbf{h}_{\theta}$, and local isotropy of embeddings around $\mathbf{e}_y$—we obtain the following result:
\begin{lemma}[Optimal Solution of Maximum Likelihood Estimation] \label{lemma:optimal_solution_mle}
    Under Assumptions \ref{assumption:unit_norm} and \ref{assumption:local_istropy}, the optimal solution of the Generic MLE Pipeline (Definition~\ref{def:generic_mle_pipeline}) satisfies
    \begin{align}
        E_Y = \argmin_{H_{\theta}} CE_Y(\theta) = \argmin_{H_{\theta}} \mathbb{E}[-\log (\text{softmax}(E^\top H_{\theta}))]
    \end{align}
\end{lemma}
The proof is provided in Appendix~\ref{appendix:proofs}. This lemma suggests that most LMs, including DLMs, effectively train their backbones to predict the target embedding sequence. Because these embeddings are static with respect to neighboring words\footnote{Here, “static” does not imply that embeddings or logit weights are fixed during training.}, the framework tends to emphasize local, word-level accuracy over sentence-level global structure. In AR models, this effect is mitigated by conditioning on ground-truth prefixes, which provide rich contextual information and implicitly capture inter-token dependencies. In contrast, under noisy or imperfect inputs—as in non-AR settings—this tendency is amplified, exacerbating the multimodality problem by favoring accurate single-word predictions without sufficient coordination across tokens.

\begin{figure}[]
    \makebox[\textwidth][c]{%
    \includegraphics[width=1.2\textwidth]{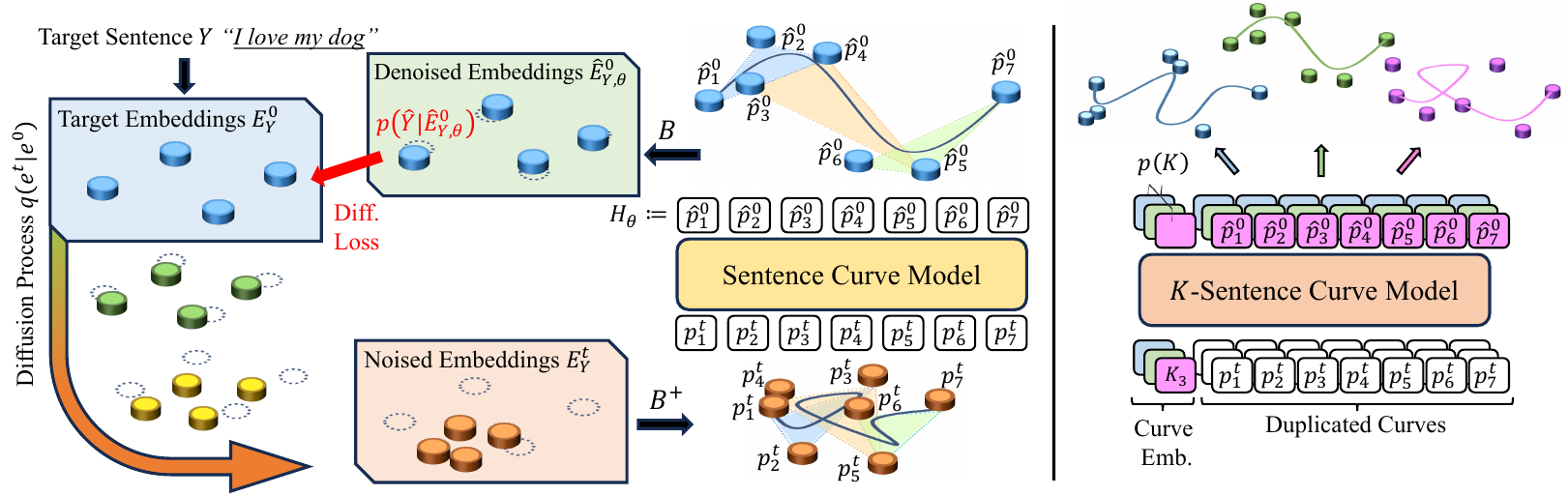}
    }
    \caption{
    (Left) SCLM training at step $t$: target embeddings are noised, mapped to a sentence curve, denoised by the backbone model, and mapped back to embeddings
    ; our modification is limited to these input/output curve mappings. (Right) Illustration of $K$-sentence curve prediction.
    }
    \label{fig:sclm_overview}
\end{figure}

\section{Sentence Curve}

In this section, we introduce sentence curves, motivated by time-series probabilistic curve modeling \cite{hug2020introducing}. We first formalize sentence curves, then incorporate sentence curve prediction into DLMs to obtain sentence curve language models (Figure~\ref{fig:sclm_overview}). Finally, we explain how this formulation promotes sentence-level global structure modeling while mitigating overfitting to word-level local structure.

\subsection{Definition of Sentence Curve}
\label{subsec:sentence_curve_definition}

We construct sentence curves using B-spline functions. A B-spline curve is defined as $\mathbf{z}(\gamma) = \sum_{j=1}^{N} b_{\eta}(\gamma)_j \mathbf{p}_j \in \mathbb{R}^{d}$, where $P = [\mathbf{p}_j]_{j=1}^{N} \in \mathbb{R}^{d \times N}$ denotes the control points and $N$ is their number. The basis function $b_{\eta}(\gamma) \in [0,1]^N$, satisfying $\sum_j b_{\eta}(\gamma)_j = 1$, is a degree-$\eta$ B-spline basis that determines the contribution of neighboring control points to each curve segment. 

\begin{wrapfigure}{r}{0.5\textwidth}
    \centering
    \vspace{-0pt}
    \includegraphics[width=0.5\textwidth]{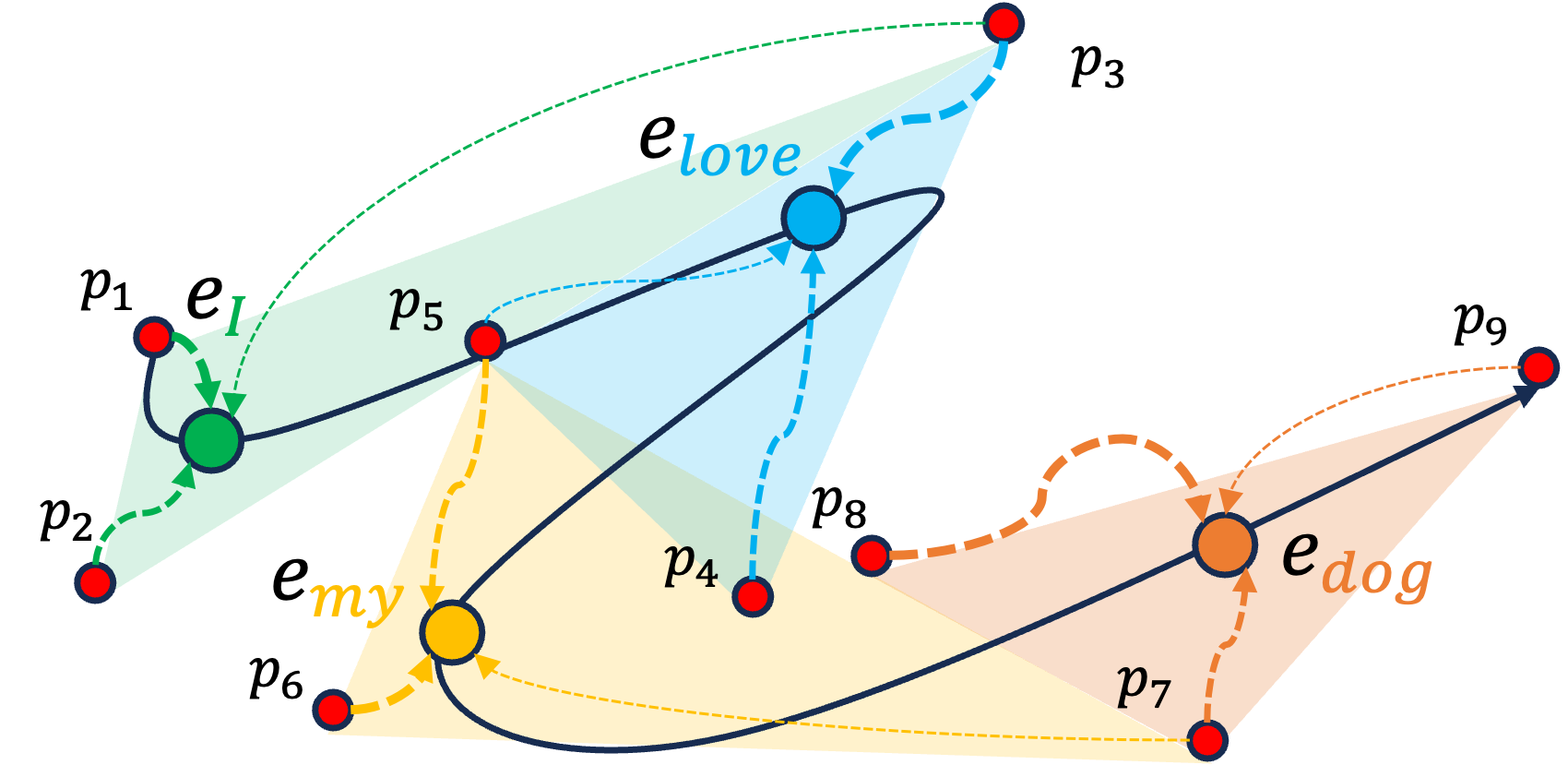}
    \caption{Illustration of control points and embeddings for the sentence \textit{“I love my dog”} with $\eta=3$. Small red dots denote control points ($\mathbf{p}$), and larger dots denote word embeddings ($\mathbf{e}$). The black curve is the spline defined by the control points, while colored dotted lines indicate their contributions to embeddings, weighted by $b_{\eta,j}$.
    }
    \label{fig:P_E_relationship}
    \vspace{-0pt}
\end{wrapfigure}
The parameter $\gamma$ denotes the curve index at which $\mathbf{z}(\gamma)$ is evaluated. Given a sentence of length $L$, we sample $L$ points at predefined indices $\Gamma = (\gamma_1, \ldots, \gamma_L)$, yielding $Z_\Gamma=\left[\mathbf{z}(\gamma_1),\cdots,\mathbf{z}(\gamma_L)\right]\in \mathbb{R}^{d\times L}$. We define these sampled points as the embedding sequence, i.e., $E_Y = Z_{\Gamma}$. Let $B_{\Gamma,\eta}:=[b_{\eta}(\gamma_1),\cdots,b_{\eta}(\gamma_L)]\in [0,1]^{N\times L}$. Then,
\begin{align}
    E_Y=P_Y B_{\Gamma,\eta}, \label{eq:p_e_relationship}
\end{align}
where $B_{\Gamma,\eta}$ is a linear mapping from control points to embeddings. We refer to $P_Y$ as the \textit{sentence curve} of $Y$. As illustrated in Figure~\ref{fig:P_E_relationship}, each control point contributes to multiple embeddings with different weights. Consequently, modifying a single embedding affects multiple control points, implying that control points may embed contextual, sentence-level structure.

In addition to the forward mapping $P \rightarrow E$, we define an approximate inverse $E \rightarrow P$ to match the number of control points for backbone inputs. Since $B_{\Gamma,\eta} \in [0,1]^{N \times L}$ is generally non-invertible when $N \neq L$, we use its pseudoinverse $B^{+} := (B^\top B)^{-1} B^\top$, yielding $P_Y \approx E_Y B^{+}$. This provides a least-squares solution consistent with the forward mapping. For brevity, we omit $\Gamma$ and $\eta$.

\subsection{Sentence Curve Language Models}
\label{subsec:sentence_curve_language_model}

Based on the definition of sentence curves (Section~\ref{subsec:sentence_curve_definition}), we propose a modified MLE pipeline:
\begin{definition}[Sentence Curve Prediction-based MLE Pipeline] \label{def:sentence_curve_prediction_pipeline}
\begin{align}
\underbrace{H_{\theta} \rightarrow H_{\theta}B}_{\text{sentence curve mapping}} \rightarrow E^\top (H_{\theta}B) \rightarrow \text{softmax}(E^\top (H_{\theta}B)) \rightarrow \min_{\theta} CE_Y(\theta). \nonumber
\end{align}
\end{definition}
This pipeline inserts a linear mapping $H_{\theta} \rightarrow H_{\theta}B$, replacing $H_{\theta}$ with $H_{\theta}B$ in the generic MLE pipeline (Definition~\ref{def:generic_mle_pipeline}). By Lemma~\ref{lemma:optimal_solution_mle}, the optimal solution satisfies $H_{\theta}B = E_Y = P_Y B$, implying that $H_{\theta}$ corresponds to the control points $P_Y$. Thus, sentence curve prediction trains the backbone to predict sentence curves rather than static word embeddings.

We incorporate this into DLMs to construct SCLMs, focusing on Gaussian DLMs (the same applies to discrete DLMs). Given clean embeddings $E_Y^0$ and their noised version $E_Y^t$, we first apply $E \rightarrow P$ using $B^{+}$ to obtain $P_Y^t$. The backbone processes $P_Y^t$ to produce $H_{\theta}(P_Y^t, t)$, which is mapped back via $P \rightarrow E$ using $B$ to obtain $\hat{E}_{Y,\theta}^0$ for loss computation. The overall process is illustrated in Figure~\ref{fig:sclm_overview} and further explained in Algorithm~\ref{alg:sclm}.
Overall, SCLM introduces two mappings—$E \rightarrow P$ at the input and $P \rightarrow E$ at the output—without modifying the backbone architecture. The same design applies to discrete DLMs by inserting these mappings before and after the backbone. 

While sentence curve prediction mitigates the token-level multimodality problem by encouraging coordinated predictions across tokens, it might not fully resolve multimodality at the sentence level. In particular, when multiple valid target sentence curves exist, a single predicted curve may represent an average of them. To address this, we propose $K$-sentence curve prediction, which predicts $K$ candidate curves. Specifically, we duplicate the input curve, prepend a learnable curve-specific token to each copy, and obtain $K$ predicted curves along with their embeddings. A shallow network estimates their probabilities. During training, we use a probability-weighted combination of the curves, while at inference we select the most probable one. The process of $K$-sentence curve prediction method is illustrated in the right side of Figure~\ref{fig:sclm_overview} and described in Algorithm~\ref{alg:k_sclm}.

\subsection{Theoretic Studies of Sentence Curve Prediction}
\label{subsec:theoretic_studies}

In this section, we analyze how sentence curve prediction regularizes the backbone to emphasize global structure, mitigating overfitting to word-level local structure. All proofs and full descriptions of following lemmas are provided in Appendix~\ref{appendix:proofs}. We begin by newly deriving the MLE objective:
\begin{lemma} [Sentence Curve Prediction Objective (Simplified)]
    Based on the setting in Section~\ref{subsec:sentence_curve_language_model}, the MLE objective can be expressed as
    \begin{align}
        CE_Y(\theta) = CE_P(\theta) - \mathbb{E}_{Y}[KL_{P|Y}(\theta)] + \mathcal{C}, \label{eq:mle_sentence_curve}
    \end{align}
    where $\mathcal{C}$ denotes terms constant with respect to $\theta$.
\end{lemma}
This decomposition shows that minimizing $CE_Y(\theta)$ implicitly minimizes $CE_P(\theta)$, encouraging the model to assign high likelihood to valid sentence curves that map to $Y$. Because each control point influences multiple words, this objective promotes sensitivity to sentence-level structure.

We next characterize how $CE_P$ emphasizes different error patterns:
\begin{lemma} [Error Importance Characterized by $B$ (Simplified)] \label{lemma:sclm_dynamics}
Let the embedding error be $V := E_Y - \hat{E}_{Y,\theta} \in \mathbb{R}^{d \times L}$ and $M := (B^\top B)^{-1}$. Defining $I(V) := \| V M^{1/2} \|_F^2$, $CE_P(\theta)$ and relative importance of global versus local errors $R(B)$ satisfy
\begin{align}
    CE_P(\theta) \propto \| V M^{1/2} \|_F^2, \quad \text{and} \quad R(B) = \frac{I(V_{\text{global}})}{I(V_{\text{local}})} \leq \frac{s_{\max}^2}{s_{\min}^2}, \nonumber
\end{align}
where $s_{\max}$ and $s_{\min}$ are the largest and smallest singular values of $B$, respectively. Here, $V_{\text{global}} = \mathbf{1}_{d \times L}/\sqrt{dL}$ and $V_{\text{local},i} = \mathbf{1}_d \delta_i^\top$ denote global and $i$-th local errors.
\end{lemma}
This result shows that error importance is shaped by $M=(B^\top B)^{-1}$: SCLMs emphasize directions aligned with the structure induced by $B$. In particular, when $B$ has dominant singular values, global errors are amplified relative to local ones, encouraging coherent sentence-level behavior. Conversely, when singular values are uniform, this distinction diminishes; in the limiting case where $B$ is identity, SCLM reduces to a standard LM with no global preference.


Notably, multiple sentence curves can correspond to the same target sentence. The set of valid curves for a given $Y$ forms a fiber set $\mathcal{F}_Y=\{P|PB=E_Y\}$, from which the model should select an appropriate sentence curve. The internal posterior $p(P|Y,X)$ induced by the second term in Eq.~\ref{eq:mle_sentence_curve} (i.e., $KL_{P|Y}(\theta) = \mathbb{E}_{p_{data}(P|Y,X)}[\log \frac{p_{data}(P|Y,X)}{p_{\theta}(P|Y,X)}]$) represents the probability of a sentence curve $P$ given $Y$ and $X$, and can be interpreted as additional supervision over the fiber beyond what is inferred from $X$ alone. Maximizing this KL term regularizes the model against strictly adhering to a particular, potentially suboptimal mapping from $Y$ to $P$, allowing it to explore the space of valid fibers and select more suitable representations. In practice, the true data distribution over fibers is unknown and difficult to estimate. Our use of the approximate inverse mapping $B^{+}$ (Section~\ref{subsec:sentence_curve_definition}) selects a representative element from each fiber, but may introduce bias and limit diversity. In this setting, the KL term acts as a regularizer that discourages collapse to a single fiber and helps preserve generalization across diverse sentence curve predictions.

\begin{table}[t]
\centering
\small
\caption{
SacreBLEU ($\uparrow$) results on IWSLT14 and WMT14. Among DLMs, the best and second-best results are shown in \textbf{bold} and \uline{underlined}, respectively. Scores without superscripts are taken from original papers, except for Transformer and Difformer, which are evaluated by us. Superscripts */\dag/\ddag/\S~ denote results reported in prior work \cite{ye2023diffusion, gao2024empowering, wu2023ar, arriola2025encoder}.}
\label{table:sacrebleu_results}
\begin{tabular}{lcccc}
\toprule
\textbf{Model} &
\makecell{\textbf{IWSLT} \textbf{En$\rightarrow$De}} &
\makecell{\textbf{IWSLT} \textbf{De$\rightarrow$En}} &
\makecell{\textbf{WMT} \textbf{En$\rightarrow$De}} &
\makecell{\textbf{WMT} \textbf{De$\rightarrow$En}} \\
\midrule
AR Transformer (beam=1) & 27.19 & 32.89 & 26.57 & 31.72  \\
AR Transformer (beam=5) & 28.24 & 33.83 & 27.71 & 33.00  \\
\addlinespace
\multicolumn{5}{l}{\textit{DLM baselines}} \\
DiffusionLM \cite{li2022diffusion}\textsuperscript{*} & - & 29.11 & 17.41 & - \\
CDCD \cite{dieleman2022continuous}\textsuperscript{\dag} & - & - & 19.7 & 25.4  \\
SeqDiffuSeq \cite{yuan2022seqdiffuseq} & 22.12 & 30.45 & 19.76 & 23.93  \\
DINOISER \cite{ye2023dinoiser} & 26.14 & 31.61 & 25.88 & \uline{30.30}  \\
GENIE \cite{lin2023text}\textsuperscript{\ddag} & 23.89 & 29.45 & - & -  \\
AR-Diffusion \cite{wu2023ar} & 26.01 & 31.80 & - & -  \\
RDM \cite{zheng2023reparameterized}\textsuperscript{*} & - & 32.14 & 26.54 & -  \\
MDLM \cite{sahoo2024simple}\textsuperscript{\S} & - & - & 18.4 & -  \\
Difformer \cite{gao2024empowering} w/o KD & 22.99 & 28.58 & - & - \\
Difformer \cite{gao2024empowering} w/ KD & \uline{27.21} & 31.19 & \uline{27.13} & 30.02 \\
WDR \cite{heo2024n} & 26.26 & 31.83 & - & -  \\
E2D2 \cite{arriola2025encoder} & - & - & 24.8 & -  \\
NeoDiff \cite{li2025unifying} & - & \uline{32.20} & 24.64 & -  \\
\midrule
\textbf{SCLM (ours) w/o KD} & 26.26 & 31.25 & - & - \\
\textbf{SCLM (ours) w/ KD} & \textbf{28.10} & \textbf{32.56} & \textbf{28.19} & \textbf{33.49} \\
\bottomrule
\end{tabular}
\vspace{-5pt}
\end{table}

\section{Experiments and Results}

Our experiments consist of two parts: translation and generic language modeling. For translation, we compare SCLM with Gaussian DLMs on standard benchmarks, providing both quantitative results and qualitative analyses that demonstrate improved global structure modeling and mitigation of the multimodality problem. For language modeling, we evaluate SCLM on LM1B against discrete DLMs (e.g., MDLM and SoftMask) to assess its effectiveness in unconditional generation.

\subsection{Datasets}

For DLM-based machine translation experiments, we use IWSLT14 En–De \cite{cettolo2014report} and WMT14 En–De \cite{bojar2014findings} datasets. We produced KD dataset for IWSLT14 by our AR Transformer model, and we adopt the publicly available KD dataset for WMT14 \cite{gao2024empowering}. For generic language modeling experiments, we use the LM1B dataset \cite{chelba2013one}. More details are provided in Appendix~\ref{appendix:dataset_detail}.

\subsection{Implementation and Configurations}
\label{subsec:implementation_configuration}

SCLM introduces several tunable hyperparameters, including the number of control points $N$, curve degree $\eta$, and the number of predicted sentence curves $K$. We set $N=\mathrm{round}(L N_{\text{ratio}})$ to adapt to sentence length $L$, and determine $\eta$ either by $\eta=\max(\mathrm{round}(N\eta_{\text{ratio}}),2)$ or by fixing $\eta=\eta_{\text{fix}}$. Curve sampling indices $\Gamma$ are defined as $L$ uniformly spaced points in $[0.01, 0.99]$. For each configuration, the matrices $B$ and $B^{+}$ are determined; to reduce overhead, we precompute them for all settings and sentence lengths $L \in \{2, \ldots, 250\}$.

We build SCLMs on top of Difformer \cite{gao2024empowering} and SoftMask \cite{hersche2025soft}, widely used baselines for conditional text generation and unconditional language modeling, respectively. Detailed configurations are provided in Appendix~\ref{appendix:dlm_experiment_detail}, with SCLM-specific hyperparameters listed in Table~\ref{table:sclm_hparams}. An ablation study on IWSLT14 De$\rightarrow$En is shown in Table~\ref{table:ablation_study}, where hyperparameters are selected based on validation performance. Notably, SCLM introduces no additional parameters in most settings, with only marginal overhead in the $K$-sentence curve variants.

\subsection{Machine Translation Experiment Results}

For translation evaluation, we use SacreBLEU \cite{post2018call}\footnote{Signature:\texttt{BLEU+case.mixed+numrefs.1+smooth.exp+tok.13a+version.1.5.1}} on the IWSLT14 and WMT14 benchmarks. We report scores from prior DLM studies on the same datasets, while Transformer and Difformer are evaluated under our setting. Table~\ref{table:sacrebleu_results} shows that SCLMs outperform previous DLM methods, with `SCLM w/ KD' achieving SOTA results across all benchmarks. Compared to the development baseline, Difformer, SCLM improves performance by 0.89/1.37 on IWSLT14 En$\rightarrow$De/De$\rightarrow$En and 1.06/3.47 on WMT14 En$\rightarrow$De/De$\rightarrow$En, respectively. Moreover, on the larger WMT14 benchmark, SCLMs not only surpass the AR Transformer (beam=5) in performance, but also achieve faster inference, providing up to 38\% speedup (refer Table~\ref{table:latency} in Appendix~\ref{appendix:dlm_experiment_detail}).
\begin{wrapfigure}{r}{0.45\textwidth}
    \centering
    \vspace{-0pt}
    \includegraphics[width=0.5\textwidth]{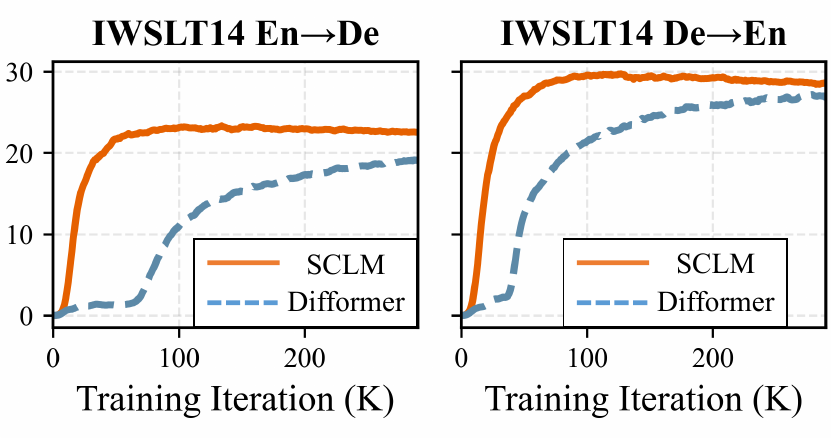}
    \caption{Validation SacreBLEU curves of Difformer and SCLM without KD.
    }
    \label{fig:wo_KD_validation}
    \vspace{-10pt}
\end{wrapfigure}

Figures~\ref{fig:front_page} (right) and~\ref{fig:extra_curve_generation_examples} illustrate sentence curve generation trajectories. Despite the inherent smoothing property of B-spline curves, SCLM produces expressive curves with dynamic inflection points that closely align with reference embeddings.

Notably, `SCLM w/o KD' achieves competitive performance against prior methods while significantly outperforming `Difformer w/o KD'. It also converges faster, as shown in Figure~\ref{fig:wo_KD_validation}. For a fair comparison, `Difformer w/o KD' is trained for 500K iterations to ensure sufficient convergence, with corresponding test scores reported in Table~\ref{table:sacrebleu_results}. We attribute these gains to improved global structure modeling in SCLM, which mitigates the multimodality problem and reduces reliance on KD.

\begin{table}[t]
\centering
\small
\caption{Generic language modeling results on LM1B. Each entry for `NFEs' and `Average' reports Gen. PPL ($\downarrow$) / self-BLEU ($\uparrow$) under different numbers of function evaluations (NFE).}
\label{table:lm1b_results}
\begin{tabular}{l|c|cccc|c}
\toprule
\textbf{Model} & \textbf{ELBO} ($\downarrow$) & \textbf{NFE=64} & \textbf{NFE=128} & \textbf{NFE=256} & \textbf{NFE=512} & \textbf{Average} \\
\midrule
MDLM & 34.55 & 63.76 / 0.08 & 38.52 / 0.13 & 30.74 / 0.16 & 24.41 / 0.20 & 39.36 / 0.14 \\
SoftMask & 34.04 & \textbf{52.62} / 0.09 & 36.07 / 0.11 & 26.14 / 0.13 & 20.60 / 0.15 & 33.86 / 0.12 \\
\midrule
\textbf{SCLM (ours)} & 33.83 & 53.10 / 0.08 & \textbf{32.51} / 0.09 & \textbf{24.82} / 0.11 & \textbf{19.43} / 0.15 & \textbf{32.47} / 0.11 \\
\bottomrule
\end{tabular}
\vspace{-10pt}
\end{table}

\vspace{-0pt}
\subsection{Language Generation Experiment Results}

For generic language generation evaluation, we use generative perplexity (Gen. PPL) and self-BLEU \cite{zhu2018texygen} to measure sample quality and diversity, respectively. Following prior work, we compute Gen. PPL using a pre-trained \texttt{GPT2-Large} \cite{radford2019language}. We compare SCLM against MDLM and SoftMask on the LM1B benchmark, using 256 generated samples from each model with the re-masking strategy \cite{wang2025remasking}. Table~\ref{table:lm1b_results} reports validation PPL (ELBO, upper bound) together with generation metrics under different numbers of function evaluations (NFE). Except for the `NFE=64' setting, SCLM achieves lower Gen. PPL than the baselines while also obtaining lower self-BLEU, indicating that it provides not only higher-quality generations but also greater diversity in generic language modeling as illustrated in Figure~\ref{fig:front_page} (bottom-left).

\subsection{Global Structure Modeling Analyses}

\begin{wrapfigure}{r}{0.37\textwidth}
    \centering
    \vspace{-35pt}
    \includegraphics[width=0.42\textwidth]{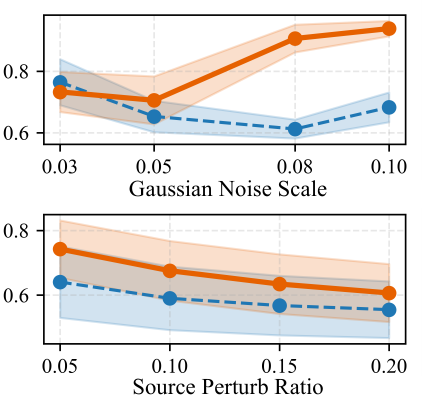}
    \vspace{-20pt}
    \caption{Distance correlation results of \textcolor{RoyalBlue}{Difformer} and our \textcolor{Bittersweet}{SCLM}. 
    }
    \label{fig:correlation_analysis}
    \vspace{-10pt}
\end{wrapfigure}
In this section, we provide evidence that SCLM enhances global structure modeling and mitigates the multimodality problem on models trained on IWSLT14. 
Firstly, we inject noise into the encoder on the valid set and measure how coherently target tokens respond in terms of distance correlation \cite{szekely2007measuring}. We consider two noises: (1) additive Gaussian noise to encoder hidden states, and (2) random masking of source tokens. The Gaussian noise is rescaled by the average hidden state norm to avoid trivial scale effects. In Figure~\ref{fig:correlation_analysis}, SCLM consistently yields higher correlations than the baseline, indicating more coherent token responses to injected noises.

Secondly, we present translation examples from SCLM and Difformer (w/o KD) in Table~\ref{table:translation_examples}. Difformer often produces phrases with mixed incompatible fragments (e.g., \textit{“any particular surface at any surface”} in Example 1) or grammatical inconsistencies (e.g., conflicting constructions in Example 3), implying the multimodality problem. In contrast, SCLM generates more coherent outputs with improved lexical and syntactic consistency.

\begin{table*}[t]
\centering
\small
\caption{Comparison of translation outputs of our SCLM and Difformer trained without KD.}
\label{table:translation_examples}
\begin{tabular}{l p{0.27\textwidth} p{0.3\textwidth} p{0.27\textwidth}}
\toprule
\textbf{} & \textbf{Ground-truth} & \textbf{Difformer} & \textbf{SCLM (ours)} \\
\midrule

1 &
\textit{
... project a light onto any surface and then , with a laser pointer , draw on it ...} &
\textit{
... project ster down a light on \textcolor{red}{any particular surface at any surface}, and and then drawing with with a laser pointer on it ...
} &
\textit{
... project a light \textcolor{red}{on top of any surface} , and then draw up a laser pointer on it ...
} \\

\addlinespace
\midrule

2 & 
\textit{Imagine you can have everybody make a small donation for one pixel .} & 
\textit{Imagine that anyone else can \textcolor{red}{give giving} a small donation to a point of image .} & 
\textit{Imagine that anyone can \textcolor{red}{give} a little donation for a picture .} \\

\addlinespace
\midrule

3 & 
\textit{But there are over 48 hours of video uploaded to youtube every minute .} & 
\textit{But on youtube , they are uploaded more than 48 hours a minute .} & 
\textit{But on youtube , there's over 48 hours of video a minute .} \\

\bottomrule
\end{tabular}
\vspace{-10pt}
\end{table*}

\section{Related Works}

\subsection{Curves in Machine Learning}

Spline-based functions (e.g., Bézier curves and B-splines) have been applied across various machine learning domains. In computer vision, curve-based representations are integrated into neural networks for tasks such as graph and 3D meshing \cite{fey2018splinecnn} and lane detection \cite{feng2022rethinking}, where they serve as inductive priors. Similarly, our sentence curve acts as an inductive prior over latent sentence structure.

Our approach is most closely related to the probabilistic curve \cite{hug2020introducing}, which applies Bézier curves to multi-step time-series prediction. In contrast, we adopt the general B-spline formulation, enabling flexible control over regularization, and explicitly incorporate an inverse mapping via a pseudoinverse, allowing seamless integration with DLMs. To our knowledge, this is the first work to connect spline curves with language modeling. Despite the discrete and symbolic nature of text, our results show that neural models can discover proper curve representations in high-dimensional embedding space.

\subsection{Sentence Target Reformulation and Why Curves?}

Few prior works explore alternative target representations for sentences. Word difference representations (WDR) \cite{heo2024n} aim to diversify targets using neighboring words, but remain auxiliary to standard embedding prediction within a multi-task framework. In contrast, our method replaces word-level targets with a curve-based representation, yielding stronger improvements.

Contextual encoders such as BERT \cite{devlin2019bert} and modern LLMs provide powerful representations that capture rich semantic information. However, they are designed for prediction rather than reconstruction and are not guaranteed to be uniquely decodable back to the original sentence. As a result, semantically similar sentences may map to similar representations, making them unsuitable as target representations. Moreover, enabling decodability would require large decoders, making backpropagation prohibitively expensive. In contrast, our sentence curve formulation preserves decodability through an explicit mapping from control points to embeddings, allowing multiple valid curves to represent the same sentence without collapsing them into a single representation. This avoids information entanglement while enabling flexible sentence-level modeling with efficient backpropagation.

\section{Conclusion and Discussion}

We propose sentence curves, a continuous, curve-based representation of sentences, and introduce sentence curve language models (SCLMs) that use them as targets. We show that sentence curve prediction encourages modeling of sentence-level global structure, and demonstrate consistent improvements in machine translation and language modeling, supported by both theoretical analysis and empirical evidence.

As an initial study, our approach leaves several directions for future work. First, the use of a pseudoinverse mapping may introduce numerical instability and information loss, particularly for longer sentences (Appendix~\ref{appendix:information_preservation_mapping}). However, since global structure modeling primarily depends on the choice of $B$ at the target side (Lemma~\ref{lemma:sclm_dynamics}), the pseudoinverse is not essential, and more stable alternatives such as regularized inverses or soft-copy mechanisms may be adopted. Second, although SCLM retains substantial speed advantages over AR Transformers (Table~\ref{table:latency}), using more control points may increase computational cost, especially for long contexts. These challenges are closely related to sequence length. In this regard, recent block-wise diffusion language models \cite{arriola2025block}, which achieve strong efficiency and quality under constrained contexts (e.g., $\leq 128$ tokens), provide a promising direction. We expect SCLM to integrate well with such frameworks, mitigating both numerical and computational issues.

\begin{ack}
\end{ack}

\bibliographystyle{unsrt} 
\bibliography{my_bib}


\appendix
\section{Mathematic Backgrounds}

\subsection{Assumptions}

\begin{assumption} [Unit Norm Constraint] \label{assumption:unit_norm}
    \begin{align}
        \|\mathbf{e}_i\|^2=\|\mathbf{h}\|^2=1 \quad \forall i\in \{1,\cdots,|\mathcal{V}|\}. \nonumber
    \end{align}
\end{assumption}
This assumption constrains all embedding vectors and the backbone output to lie on the unit sphere $\mathbb{S}^{d-1}$, allowing the analysis to focus on vector orientations rather than trivial scale differences. This normalization has also been shown to be practically effective in prior work \cite{meng2019spherical, liu2020normalization}.

\begin{assumption} [Local Isotropy \cite{arora2016latent} on the Target Embedding's Tangent Space] \label{assumption:local_istropy}
    Given a target word $y$,
    \begin{align}
        \mathbb{E}[\mathbf{u}_{i}] &= 0, \quad \mathbf{u}_{i} = \mathbf{e}_i-(\mathbf{e}_y^{\top}\mathbf{e}_i)\mathbf{e}_y \quad \forall i\in \{1,\cdots,|\mathcal{V}|\}. \nonumber
    \end{align}
\end{assumption}
This assumption states that, conditioned on a target word $y$, the remaining word embeddings are isotropically distributed when projected onto the tangent space of $\mathbf{e}_y$. Although this assumption may appear strong, it is considerably weaker than commonly adopted alternatives such as orthonormal logit bases \cite{ildiz2024self}, which require $d \geq |\mathcal{V}|$ and impose rigid structural constraints. In contrast, our assumption permits $d < |\mathcal{V}|$, as well as semantic similarity and clustered embedding structures.

\subsection{Structural Equivalence of Continuous Relaxation and General MLE Pipeline}
\label{appendix:continuous_relaxation}

In this section, we explain the connection between the word and its embedding in the context of continuous relaxation \cite{jang2016categorical}. In continuous relaxation, the embedding vector can be regarded as the mean of approximated word in the continuous space. Our next proposition says the details:
\begin{proposition} [Continuous Relaxation of Words]
    With defining $\mathbf{z}_y=E\delta_y+\epsilon$ where $\epsilon \sim \mathcal{N}(\mathbf{0},\sigma^2\mathbf{I}_d)$,
    \begin{align}
        p(\mathbf{z}_y|y) &= \mathcal{N}(\mathbf{e}_y,\sigma^2\mathbf{I}), \nonumber \\
        p(y|\mathbf{z}_y) &= \text{softmax}(E^\top\mathbf{z}_y /\sigma^2), \nonumber
    \end{align}
    conditioned on Assumption \ref{assumption:unit_norm} and when $E^\top:\mathbb{R}^{|\mathcal{V}|}\rightarrow \mathbb{R}^{d}$ is injective.
\end{proposition}
\begin{proof}
    The likelihood distribution, $p(\mathbf{z}_y|y)$ is determined by the definition. By Bayes theorem,
    \begin{align}
        p(y|\mathbf{z}_y) &= \frac{p(\mathbf{z}_y|y)p(y)}{\sum_{k=1}^{|\mathcal{V}|}p(\mathbf{z}_y|y=k)p(y=k)} = \frac{p(y)e^{-\frac{1}{2\sigma^2}\|\mathbf{z}_y-\mathbf{e}_y\|^2}}{\sum_{k=1}^{|\mathcal{V}|}p(y=k)e^{-\frac{1}{2\sigma^2}\|\mathbf{z}_y-\mathbf{e}_k\|^2}}. \nonumber
    \end{align}
    Based on Assumption~\ref{assumption:unit_norm}, $\|\mathbf{z}_y-\mathbf{e}_k\|^2 = 2-2\mathbf{e}_k^\top \mathbf{z}$. With plugging this into the above formulation:
    \begin{align}
        p(y|\mathbf{z}_y) &\approx \frac{p(y)e^{(\mathbf{e}_y^\top\mathbf{z}_y-1)/\sigma^2}}{\sum_{k=1}^{|\mathcal{V}|}p(y=k)e^{(\mathbf{e}_k^\top\mathbf{z}_y-1)/\sigma^2}} = \frac{p(y)e^{\mathbf{e}_y^\top\mathbf{z}_y/\sigma^2}}{\sum_{k=1}^{|\mathcal{V}|}p(y=k)e^{\mathbf{e}_k^\top\mathbf{z}_y/\sigma^2}} \nonumber \\
        &\approx \frac{e^{\mathbf{e}_y^\top\mathbf{z}_y/\sigma^2}}{\sum_{k=1}^{|\mathcal{V}|}e^{\mathbf{e}_k^\top\mathbf{z}_y/\sigma^2}} = \text{softmax}(E^\top\mathbf{z}_y/\sigma^2), \nonumber
    \end{align}
    with assuming $p(y=k) \quad \forall k\in \{1,\cdots,|\mathcal{V}|\}$ is constant.
\end{proof}
In other words, $\mathbf{z}_y$ is an approximation of the word $y$ in $d$-dimensional continuous space, and the embedding vector $\mathbf{e}_y$ is the mean of $\mathbf{z}_y$. Interestingly, the likelihood distribution, $p(y|\mathbf{z}_y)$ has similar form with the generic MLE pipeline (Definition~\ref{def:generic_mle_pipeline} in Section~\ref{sec:word_embedding_prediction}). With switching $\mathbf{z}_y$ to $\mathbf{h}$, and assuming $\sigma^2=1$, the likelihood has structural equivalence between the pipeline's likelihood estimation process. In this sense, what the model is trying to do can be understood as predicting a latent variable $\mathbf{z}$ in the continuous space that is aligned to the word $y$ with mapping $E$ and $E^\top$.

\subsection{Proofs of Theorems and Lemmas}
\label{appendix:proofs}

\begin{lemma}[Optimal Solution of Maximum Likelihood Estimation] 
    Under Assumptions \ref{assumption:unit_norm} and \ref{assumption:local_istropy}, the optimal solution of the Generic MLE Pipeline (Definition~\ref{def:generic_mle_pipeline}) satisfies
    \begin{align}
        E_Y = \argmin_{H_{\theta}} CE_Y(\theta) = \argmin_{H_{\theta}} \mathbb{E}[-\log (\text{softmax}(E^\top H_{\theta}))]
    \end{align}
\end{lemma}
\begin{proof}
    Without loss of generality, we focus on the minimization of a target word $y\in Y$: $\argmin_{\mathbf{h}}-
    \log (\text{softmax}(\mathbf{e}_y^\top \mathbf{h}))$, note that we omitted $\theta$ for brevity. 
    
    To derive the optimal solution given the condition $\|\mathbf{h}\|^2=1$ (Assumption \ref{assumption:unit_norm}), we use Lagrange multiplier. In result, the objective function can be derived follows:
    \begin{align}
        L(\mathbf{h},\lambda) &= f(\mathbf{h}) - \lambda(\|\mathbf{h}\|^2-1), \nonumber \\
        \nabla_{\mathbf{h}} L(\mathbf{h},\lambda) &= \nabla_{\mathbf{h}} f(\mathbf{h})-2\lambda\mathbf{h}, \nonumber
    \end{align}
    where $f(\mathbf{h}) = -\log (\text{softmax}(\mathbf{e}_y^\top \mathbf{h})) = \mathbf{e}_y^\top\mathbf{h}-\log \sum_{k}e^{\mathbf{e}_k^\top\mathbf{h}}$ by the general pipeline (Section~\ref{sec:word_embedding_prediction} and Appendix~\ref{appendix:continuous_relaxation}). Based on this formulation, the stationary condition of Lagrange multiplier becomes:
    \begin{align}
        \nabla f(\mathbf{h}) = 2\lambda\mathbf{h}. \nonumber
    \end{align}
    Notably, this informs that $\nabla_{\mathbf{h}} f(\mathbf{h})$ is a parallel vector to $\mathbf{h}$, in turn, the tangent component of $\nabla_{\mathbf{h}} f(\mathbf{h})$ with respect to $\mathbf{h}$ should be zero. This is formulated as follows:
    \begin{align}
        \nabla_{\mathbf{h}} f(\mathbf{h}) - \left(\mathbf{h}^\top\nabla_{\mathbf{h}} f(\mathbf{h}) \right)\mathbf{h} &= 0. \nonumber
    \end{align}
    With plugging the formulation of $f(\mathbf{h})$ into the above, the stationary condition becomes:
    \begin{align}
        \mathbf{e}_y - \sum_{k}p(y=k|\mathbf{h})\mathbf{e}_k - \left(\mathbf{h}^\top \mathbf{e}_k - \mathbf{h}^\top \left(\sum_{k}p(y=k|\mathbf{h}\right) \mathbf{e}_y \right)\mathbf{h} = 0. \label{eq:stat_condition}
    \end{align}

    Now, to verify our main conjecture which is the optimality of $\mathbf{e}_y=\mathbf{h}$, we set $\mathbf{h}^*=\mathbf{e}_y$ in Eq.~\ref{eq:stat_condition}, and validate whether the condition holds or not. After plugging our conjecture, the stationary condition is derived as follows:
    \begin{align}
        \sum_{k}p(y=k|\mathbf{e}_y)\left( \mathbf{e}_k-(\mathbf{e}_y^\top\mathbf{e}_k)\mathbf{e}_y \right) &= \sum_{k, k\neq y}p(y=k|\mathbf{e}_y)\left( \mathbf{e}_k-(\mathbf{e}_y^\top\mathbf{e}_k)\mathbf{e}_y \right) \nonumber \\
        &= \mathbb{E}_{p(y=k|\mathbf{e}_y)}[\mathbf{e}_k]-\mathbb{E}_{p(y=k|\mathbf{e}_y)}[(\mathbf{e}_y^\top\mathbf{e}_k)\mathbf{e}_y]  = 0. \label{eq:stat_condition_new}
    \end{align}
    Given the target embedding $\mathbf{e}_y$, all the other embedding vector can be decomposed into parallel and tangent components with respect to $\mathbf{e}_y$: $\mathbf{e}_i=(\mathbf{e}_y^\top\mathbf{e}_i)\mathbf{e}_y + (\mathbf{e}_i - (\mathbf{e}_y^\top\mathbf{e}_i)\mathbf{e}_y)$. We denote the tangent component as $\mathbf{u}_i=\mathbf{e}_i - (\mathbf{e}_y^\top\mathbf{e}_i)\mathbf{e}_y$. Based on this decomposition and the given assumption (Assumption~\ref{assumption:local_istropy}), the expectation term becomes:
    \begin{align}
        \mathbb{E}[\mathbf{e}_k] = \mathbb{E}[(\mathbf{e}_y^\top\mathbf{e}_k)\mathbf{e}_y+\mathbf{u}_k] = \mathbb{E}[(\mathbf{e}_y^\top\mathbf{e}_k)\mathbf{e}_y]. \nonumber
    \end{align}
    This makes the stationary, Eq.~\ref{eq:stat_condition_new}, be satisfied. Consequently, our conjecture of the optimality $\mathbf{e}_y=\mathbf{h}$, which becomes $E_Y=H$ in matrix level, is proved.
    
\end{proof}

\begin{lemma} [Sentence Curve Prediction Objective]
    With assuming $Y$ and $E$ are equal in distribution, both marginally and conditionally on any variables, such as $P$ and $X$, via $\sigma^2 \rightarrow 0$ in continuous relaxation scheme (Appendix~\ref{appendix:continuous_relaxation}). Based on settings in Section~\ref{subsec:sentence_curve_definition}, the MLE objective function is derived as follows:
    \begin{align}
        CE_Y(\theta) = CE_P(\theta) - \mathbb{E}_{Y}[KL_{P|Y}(\theta)] + \mathcal{C}, \nonumber
    \end{align}
    where $\mathcal{C}$ indicates constant terms with respect to $\theta$.
\end{lemma}
\begin{proof}
    By Bayes theorem, the log-likelihood of $Y$ given $X$ is derived as follows:
    \begin{align}
        \log p_{\theta}(Y|X) &= \log p_{\theta}(Y|P,X) + \log p_{\theta}(P|X) - \log p_{\theta}(P|Y,X) \nonumber \\
        &= \log p(Y|P) + \log p_{\theta}(P|X) - \log p_{\theta}(P|Y,X), \nonumber
    \end{align}
    where the second equation is done by considering $P \xrightarrow{B} Y$ is deterministic many-to-one mapping. With applying expectation over $p_{data}(Y,P|X)$, we have that:
    \begin{align}
        \underbrace{\mathbb{E}_{p_{data}(Y|X)}[-\log p_{\theta}(Y|X)]}_{=CE_Y(\theta)} &= \underbrace{\mathbb{E}_{p_{data}(P|X)}[-\log p_{\theta}(P|X)]}_{=CE_P(\theta)} + \mathbb{E}_{p_{data}(Y,P|X)}[\log p_{\theta}(P|Y,X)] \nonumber \\
        &\qquad + \underbrace{\mathbb{E}_{p_{data}(Y,P|X)}[-\log p(Y|P)]}_{=\mathbb{E}_{P|X}[\mathcal{H}(Y|P)] \text{(constant)}} \nonumber \\ 
        &= CE_P(\theta) + \mathbb{E}_{p_{data}(Y|X)} \underbrace{\left[\mathbb{E}_{p_{data}(P|Y,X)}\left[ -\log \frac{p_{data}(P|Y,X)}{p_{\theta}(P|Y,X)} \right] \right]}_{=-KL_{P|Y}(\theta)} \nonumber \\
        &\qquad + \underbrace{\mathbb{E}_{p_{data}(Y,P|X)}[\log p_{data}(P|Y,X)]}_{=\mathbb{E}_{Y|X}[-\mathcal{H}(P|Y,X)] \text{(constant)}} + \mathcal{C} \nonumber \\
        &= CE_P(\theta) - \mathbb{E}_{Y}[KL_{P|Y}(\theta)] + \mathcal{C} \nonumber
    \end{align}
\end{proof}

\begin{lemma} [Error Importance Characterized by $B$] 
Let the prediction error be defined as $V := E_Y - \hat{E}_{Y,\theta} \in \mathbb{R}^{d \times L}$, and assume $p_{\theta}(P \mid X) = \mathcal{N}(P; \hat{P}_{\theta}, \sigma^2 \mathbf{I}_d)$ for any variance $\sigma^2$. Let $M := (B^\top B)^{-1}$. Then, the cross-entropy objective induced by sentence curve prediction satisfies
\begin{align}
    CE_P(\theta) &\propto \| V M^{1/2} \|_F^2 \nonumber \\
            &= \mathrm{tr}(V M V^\top). \nonumber
\end{align}
Defining the error importance as $I(V) := \| V M^{1/2} \|_F^2$, the relative importance of global versus local errors under $B$ is given by
\begin{align}
    R(B) := \frac{I(V_{\text{global}})}{I(V_{\text{local}})} \;\leq\; \frac{s_{\max}^2}{s_{\min}^2},
\end{align}
where $s_{\max}$ and $s_{\min}$ denote the maximum and minimum singular values of $B$, respectively. Here, $V_{\text{global}} = \mathbf{1}_{d \times L}/\sqrt{dL}$ represents a global error pattern, and $V_{\text{local},i} = \mathbf{1}_d \delta_i^\top$ denotes a localized error at position $i$.
\end{lemma}
\begin{proof}
    Based on the definition of the model's error in $E$-space, $V=E_Y-\hat{E}_{Y,\theta}$, we derive the error in $P$-space as $V_P:P_Y-\hat{P}_{\theta}=E_YB^+-\hat{E}_{Y,\theta}B^+=VB^+$. From now on, we concentrate on one dimension out of $d$ for simple derivation without loss of generality. For example, we assume $E_Y\in\mathbb{R}^{L}$, $P\in\mathbb{R}^{N}$, $V\in\mathbb{R}^{L}$, and $VB^+\in\mathbb{R}^{N}$. The negative log-likelihood can be derived as follows:
    \begin{align}
        -\log p_{\theta}(P|X) &= \frac{1}{2\sigma^2}\| P - \hat{P}_{\theta} \|^2_F = \frac{1}{2\sigma^2} \| VB^+ \|^2_F, \nonumber \\
        CE_P(\theta) &\propto \| VB^+ \|^2_F \nonumber \\
        &=\text{tr}(VB^+(B^+)^{\top}V^\top) \nonumber \\
        &=\text{tr}(V\underbrace{(B^\top B)^{-1}}_{=M}V^\top) \nonumber \\
        &=\|VM^{1/2} \|^2_F. \nonumber
    \end{align}
    
    Because $M=(B^\top B)^{-1}$ is positive semi-definite matrix, it can be eigen-decomposed as follows:
    \begin{align}
        M = \sum_{j=1}^L \lambda_j \mathbf{k}_j \mathbf{k}_j^\top, \nonumber
    \end{align}
    where $\lambda$ and $\mathbf{k}$ are eigenvalue and eigenvector of $M$, respectively. Based on those bases from $M$, we can decompose any error vector as follows:
    \begin{align}
        V &= \sum_{j=1}^L \alpha_j \mathbf{k}_j, \qquad VMV^\top = \sum_{j=1}^L\lambda_j\alpha_j^2, \nonumber
    \end{align}
    where $\alpha_j=V^\top\mathbf{k}_j$ is the coefficient for the factor $V$ with respect to $\mathbf{k}$ vector. Based on those derivations, the global and local error importance are derived as follows:
    \begin{align}
        I(V_{global})&=\| V_{global}M^{1/2}\|^2_F=\frac{1}{L}\| \mathbf{1}_{d \times L} M^{1/2} \|^2_F=\frac{1}{L}\sum_{j=1}^L\lambda_j \left( \sum_{i=1}^L\mathbf{k}_{ji} \right)^2 \nonumber \\
        &\leq \lambda_{max}\left( \sum_{i=1}^L \mathbf{k}_{max,i} \right)^2 \leq \lambda_{max} \sum_{i=1}^L\mathbf{k}_{max,i}^2 = \lambda_{max}, \nonumber \\
        I(V_{local}) &= \| V_{local,i}M^{1/2}\|^2_F = \sum_{j=1}^L \lambda_j\mathbf{k}_{jj}^2 \geq \lambda_{min}\sum_{j=1}^L\mathbf{k}_{min,j}^2=\lambda_{min}, \nonumber
    \end{align}
    where the second inequality of global error importance is derived by Jensen's inequality. With plugging in those results, we can get
    \begin{align}
        R(B) = \frac{I(V_{global})}{I(V_{local})} &\leq \frac{\lambda_{max}}{\lambda_{min}}. \nonumber
    \end{align}
    Because the relationship, $M=(B^\top B)^{-1}$, $\lambda_{max}$ and $\lambda_{min}$ can be further derived as $1/s_{min}^2$ and $1/s_{max}^2$, respectively. Finally, the relative importance given $B$ is derived as follows:
    \begin{align}
        R(B) &\leq \frac{s_{max}^2}{s_{min}^2}. \nonumber
    \end{align}
\end{proof}

    

\section{Details of Sentence Curve Language Models}
\label{appendix:detail_sclm}

\begin{algorithm}[t]
\small
\caption{Sentence Curve Language Model (SCLM) Training}
\label{alg:sclm}

\KwIn{Data $(X,Y)$, embedding matrix $E$, mappings $B,B^{+}$, DLM noise process $q$, backbone $H_{\theta}$}
\KwOut{Updated parameters $\theta$ and $E$}

Sample $(X,Y)\sim p_{\text{data}}$\;
Compute clean embeddings $E_Y^0 = E\delta_Y$\;

Sample timestep $t$ and noised representation $E^t_Y \sim q(\cdot|E_Y^0)$\;

\tcp{\textcolor{red}{SCLM input mapping}}
Map to sentence curve:
$P_Y^t \leftarrow E^t_Y B^{+}$\;

\tcp{Backbone model process}
Predict denoised curve:
$\hat{P}_{Y,\theta}^0 \leftarrow H_{\theta}(P_Y^t,t,X)$\;

\tcp{\textcolor{blue}{SCLM output mapping}}
Map back to embedding space:
$\hat{E}_{Y,\theta}^0 \leftarrow \hat{P}_{Y,\theta}^0 B$\;

\tcp{Unified DLM loss}
Compute $\mathcal{L}(Y;\theta,E)$ with Eq.~\ref{eq:gdlm_objective} or Eq.~\ref{eq:mdlm_projection_layer} (left).

Update $\theta$ and $E$\;

\end{algorithm}

\begin{algorithm}[t]
\small
\caption{$K$-Sentence Curve Prediction}
\label{alg:k_sclm}

\KwIn{Input curve $P_Y^t$, mapping $B$, backbone $H_{\theta}$, number of curves $K$}
\KwOut{Predicted embedding $\hat{E}_{Y,\theta}^0$}

Duplicate curve:
$\{P_{Y,k}^t\}_{k=1}^{K}$\;

Concatenate curve-specific tokens to each curve:
$\tilde{P}_{Y,k}^t \leftarrow [\mathbf{c}_k \,\|\, P_{Y,k}^t], \quad k=1,\ldots,K$\;

Predict curves and output tokens:
$\{(\hat{P}_{Y,k}^0,\mathbf{c}'_k)\}_{k=1}^{K}
\leftarrow H_{\theta}(\{\tilde{P}_{Y,k}^t\}_{k=1}^{K},t,X)$

\tcp{Curve probability estimation}
Compute probabilities:
$\pi_k \leftarrow \mathrm{softmax}(g(\mathbf{c}'_k))$\;

\tcp{$g(\cdot)$: shallow MLP with GeLU activation}
Map to embedding space:
$\hat{E}_{Y,k}^0 \leftarrow \hat{P}_{Y,k}^0 B$\;

\eIf{training}{
    $\hat{E}_{Y,\theta}^0 \leftarrow \sum_{k=1}^{K}\pi_k \hat{E}_{Y,k}^0$\;
}{
    $\hat{E}_{Y,\theta}^0 \leftarrow \hat{E}_{Y,k^\ast}^0,\quad
    k^\ast=\arg\max_k \pi_k$\;
}

\end{algorithm}

In this section, we describe the algorithmic formulations of SCLM and $K$-sentence curve prediction, summarized in Algorithms~\ref{alg:sclm} and~\ref{alg:k_sclm}, respectively.

\section{Experimental Details}

\subsection{Details of Datasets}
\label{appendix:dataset_detail}

In our machine translation experiments, we use the IWSLT14 English–German and WMT14 English–German datasets. Data cleaning, tokenization, and byte-pair encoding (BPE) are performed using the Fairseq toolkit \cite{ott2019fairseq}, closely following the preprocessing pipelines adopted in prior works \cite{ye2023dinoiser, gao2024empowering}. After preprocessing, the IWSLT14 dataset contains approximately 160K/7K/7K sentence pairs for the train/validation/test splits, respectively, with a vocabulary size of 10K. The processed WMT14 dataset contains approximately 4.5M/3K/3K sentence pairs with a vocabulary size of 40K. For generic language modeling experiments, we use the LM1B dataset. All preprocessing steps follow the protocol described in MDLM paper \cite{sahoo2024simple}. Notably, we do not employ sentence packing during data loading. After preprocessing, the dataset contains approximately 30M training sentences and 300K evaluation sentences. 

\subsection{Details of DLM Experiments}
\label{appendix:dlm_experiment_detail}

\begin{table}[t]
\centering
\caption{Basic model and training configurations.}
\label{table:basic_configuration}
\resizebox{1.0\textwidth}{!}{
    \begin{tabular}{l|ccc}
    \hline
    \multicolumn{1}{l}{\textbf{Hyperparameter}} &
    \textbf{IWSLT14 En--De} &
    \textbf{WMT14 En--De} &
    \textbf{LM1B} \\
    \hhline{====}
    Model / Architecture     & \texttt{transformer\_iwslt\_de\_en} & \texttt{transformer\_wmt\_en\_de} & \texttt{small DiT} \\
    \# Encoder layers         & 6 & 6 & -- \\
    \# Decoder layers         & 6 & 6 & 12 \\
    $d_{\text{model}}$        & 512 & 512 & 768 \\
    $d_{\text{ff}}$           & 1024 & 2048 & 3072 \\
    \# Attention heads        & 4 & 8 & 12 \\
    Dropout                  & 0.3 & 0.1 & 0.1 \\
    Label smoothing          & 0.1 & 0.1 & -- \\
    Batch size (sentence or tokens)  & 8K tokens & 32K tokens & 512 sentences \\
    Context length           & -- & -- & 128 \\
    Optimizer                & AdamW & AdamW & AdamW \\
    Learning rate            & $5e^{-4}$ & $5e^{-4}$ & $3e^{-4}$ \\
    Gradient clipping        & 0.0 & 1.0 & 1.0 \\
    Training iterations      & 300K & 300K & 200K \\
    Checkpoint ensemble      & 5 & 5 & EMA (0.9999 Decay) \\
    Latent dimension         & 128 & 128 & -- \\
    Diffusion steps          & 2K & 2K & 1K \\
    Rescaling factor          & 4.0 & 3.5 & -- \\
    Self-conditioning         & O & O & -- \\
    Reverse steps (testing)   & 20 & 20 & 1K \\
    Length Beam size (testing)  & 5 & 7 & -- \\
    Noise Beam size (testing)  & 2 & 3 & -- \\
    \hline
    \end{tabular}
}
\end{table}

\begin{table}[t]
\centering
\caption{Tuned hyperparameters of proposed SCLM models. If the value of $\eta$ is float or integer, then the value is ratio or fixed value, respectively.}
\label{table:sclm_hparams}
\resizebox{0.9\textwidth}{!}{
    \begin{tabular}{lccccccc}
    \toprule
    \textbf{Hyperparameter} &
    \makecell{\textbf{IWSLT14} \\ \textbf{En$\rightarrow$De}} &
    \makecell{\textbf{IWSLT14} \\ \textbf{De$\rightarrow$En}} &
    \makecell{\textbf{IWSLT14} \\ \textbf{En$\rightarrow$De} \\ \textbf{(w/o KD)}} &
    \makecell{\textbf{IWSLT14} \\ \textbf{De$\rightarrow$En} \\ \textbf{(w/o KD)}} &
    \makecell{\textbf{WMT14} \\ \textbf{En$\rightarrow$De}} &
    \makecell{\textbf{WMT14} \\ \textbf{De$\rightarrow$En}} &
    \makecell{\textbf{LM1B}} \\
    \midrule
    $N_{\text{ratio}}$ & 2.0 & 3.0  & 2.0  & 2.0  & 2.0  & 3.0  & 3.0 \\
    $\eta$             & 5   & 0.01 & 0.01 & 0.01 & 0.1  & 5  & 5 \\
    $K$                & 1   & 1    & 3    & 4    & 1    & 1  & 1 \\
    \bottomrule
    \end{tabular}
}
\end{table}

In this section, we describe the details of our development baseline models and training procedures. For the backbone architecture of Difformer, we use the \texttt{transformer\_iwslt\_de\_en} and \texttt{transformer\_wmt\_en\_de} configurations provided by the Fairseq toolkit for the IWSLT14 and WMT14 experiments, respectively, following Difformer’s diffusion-based training framework\footnote{https://github.com/zhjgao/difformer}. We closely adhere to the official training scheme of Difformer for both the baseline models and our SCLMs. The full set of configuration details is summarized in `IWSLT14 En-De' and `WMT14 En-De' columns in Table~\ref{table:basic_configuration}. We note that our minimum Bayes risk (MBR) settings, including the length beam and noise beam sizes, are minimal or consistent with those used in prior works in Table~\ref{table:sacrebleu_results}. The total number of parameters for Difformer and SCLM models is 38.9M for IWSLT14 and 53.5M for WMT14, respectively. Importantly, SCLMs introduce almost no additional parameters beyond the baseline models.

For generic language modeling, we follow the official implementation and training setup of MDLM\footnote{https://github.com/kuleshov-group/mdlm}, where the backbone is a small-sized DiT \cite{peebles2023scalable}. We re-implemented the soft-masking strategy of SoftMask \cite{hersche2025soft} on top of the MDLM implementation with the standard SoftMsak-specific hyperparameters. Due to computational constraints, we reduce the number of training iterations to 200K and the context length to 128, while keeping all other settings unchanged. A detailed summary of the architecture and training protocol is provided in the `LM1B' column of Table~\ref{table:basic_configuration}; we refer to \cite{sahoo2024simple} for additional details. As a result, both MDLM and our SCLM variant contain approximately 139M parameters.

Finally, after validation-based tuning, we select the best SCLM configurations, summarized in Table~\ref{table:sclm_hparams}. We explore $N_{\text{ratio}} \in {2.0, 2.5, 3.0}$, $\eta_{\text{ratio}} \in {0.01, 0.05, 0.1}$, $\eta_{\text{fix}} \in {5, 10}$, and $K \in {1,2,3,4}$. We find that $K>1$ is particularly effective in the w/o KD setting, whereas $K=1$ performs best otherwise.

\begin{table}[t]
\centering
\small
\caption{Inference latency (seconds/sample) on IWSLT14 and WMT14 En$\rightarrow$De, measured on a single A4000 GPU. Values in parentheses indicate relative speed compared to the AR Transformer.}
\label{table:latency}
\begin{tabular}{lcc}
\toprule
\textbf{Model} & \textbf{IWSLT14 En$\rightarrow$De} & \textbf{WMT14 En$\rightarrow$De} \\
\midrule
AR Transformer & 0.3016 (0.00\%) & 0.4020 (0.00\%) \\
\addlinespace
Difformer & 0.2191 (27.36\%) & 0.2473 (38.47\%) \\
\addlinespace
SCLM ($N_{ratio}=2.0, \eta_{ratio}=0.1, K=1$) & 0.2201 (27.01\%)) & 0.2477 (38.39\%) \\
SCLM ($N_{ratio}=2.5, \eta_{ratio}=0.1, K=1$) & 0.2298 (23.82\%) & 0.2540 (36.82\%) \\
SCLM ($N_{ratio}=3.0, \eta_{ratio}=0.1, K=1$) & 0.2271 (24.70\%) & 0.2663 (33.75\%) \\
\addlinespace
SCLM ($N_{ratio}=2.0, \eta_{ratio}=0.01, K=1$) & 0.2265 (24.90\%) & 0.2443 (39.22\%) \\
SCLM ($N_{ratio}=2.0, \eta_{ratio}=0.05, K=1$) & 0.2396 (20.58\%) & 0.2537 (36.90\%) \\
SCLM ($N_{ratio}=2.0, \eta_{fix}=5, K=1$) & 0.2316 (23.23\%) & 0.2510 (37.56\%) \\
SCLM ($N_{ratio}=2.0, \eta_{fix}=10, K=1$) & 0.2350 (22.10\%) & 0.2590 (35.57\%) \\
\bottomrule
\end{tabular}
\end{table}

Experiments on WMT14 and LM1B are conducted on RTX 4090 or 3090 GPUs, while IWSLT14 uses an A4000 GPU. We also report inference latency comparisons between Difformer and SCLMs under IWSLT14 and WMT14 settings. Table~\ref{table:latency} shows inference time (seconds per sample) for all models. With $K=1$, SCLMs incur only modest overhead, increasing latency by about 5\% on IWSLT14 and 3\% on WMT14 compared to Difformer. Latency tends to grow with larger $N_{\text{ratio}}$ due to increased counts of processing tokens. Despite this overhead, SCLMs retain substantial speed advantages over AR Transformers, highlighting their practical efficiency.

\section{Additional Experimental Analyses}

\subsection{Ablation Study on IWSLT14 De$\rightarrow$En}
\label{appendix:ablation_study}

\begin{table}[t]
\centering
\caption{Ablation study of the proposed SCLM model based on IWSLT14 De$\rightarrow$En.}
\label{table:ablation_study}
\resizebox{0.7\textwidth}{!}{
    \begin{tabular}{lcc}
    \toprule
    \textbf{Model} &
    \textbf{Valid BLEU} &
    \textbf{Test BLEU} \\
    \midrule
    \multicolumn{3}{l}{\textit{Baselines}} \\
    Transformer & 35.71 & 33.83 \\
    Difformer   & 31.81 & 31.19 \\
    \midrule
    \multicolumn{3}{l}{\textit{Default Config.}} \\
    SCLM ($N_{ratio}=2.0,\eta_{ratio}=0.1,K=1$) & 33.25 & 32.35 \\
    \midrule
    \multicolumn{3}{l}{\textit{Varying $N_{\text{ratio}}$}} \\
    SCLM ($N_{ratio}=2.0,\eta_{ratio}=0.1,K=1$) & 33.25 & 32.35 \\
    SCLM ($N_{ratio}=2.5,\eta_{ratio}=0.1,K=1$) & 33.28 & 32.04 \\
    SCLM ($N_{ratio}=3.0,\eta_{ratio}=0.1,K=1$) & \textbf{33.63} & 32.41 \\
    \midrule
    \multicolumn{3}{l}{\textit{Varying $\eta$}} \\
    SCLM ($N_{ratio}=3.0,\eta_{ratio}=0.01,K=1$) & \textbf{33.79} & 32.56 \\
    SCLM ($N_{ratio}=3.0,\eta_{ratio}=0.05,K=1$) & 33.50 & 32.35 \\
    SCLM ($N_{ratio}=3.0,\eta_{ratio}=0.1,K=1$)  & 33.63 & 32.41 \\
    SCLM ($N_{ratio}=3.0,\eta_{fix}=5,K=1$)  & 33.37 & 32.25 \\
    SCLM ($N_{ratio}=3.0,\eta_{fix}=10,K=1$) & 33.11 & 32.13 \\
    \midrule
    \multicolumn{3}{l}{\textit{Varying $K$}} \\
    SCLM ($N_{ratio}=3.0,\eta_{ratio}=0.01,K=1$) & 33.79 & 32.56 \\
    SCLM ($N_{ratio}=3.0,\eta_{ratio}=0.01,K=2$) & 33.57 & 32.51 \\
    SCLM ($N_{ratio}=3.0,\eta_{ratio}=0.01,K=3$) & 33.60 & 32.30 \\
    SCLM ($N_{ratio}=3.0,\eta_{ratio}=0.01,K=4$) & 33.65 & 32.18 \\
    \bottomrule
    \end{tabular}
}
\end{table}

In this section, we present an ablation study that illustrates our hyperparameter tuning process based on validation performance. We conducted a grid search over the predefined ranges of each hyperparameter described in Section~\ref{appendix:dlm_experiment_detail}. Table~\ref{table:ablation_study} reports the ablation results on the IWSLT14 De$\rightarrow$En task. As shown, we select the configuration that achieves the highest validation BLEU score. Notably, we observe consistent performance improvements over the baseline.

\subsection{Information Preservation of Sentence Curve Mapping}
\label{appendix:information_preservation_mapping}

\begin{figure}[]
    \centering
    \includegraphics[width=1.0\linewidth]{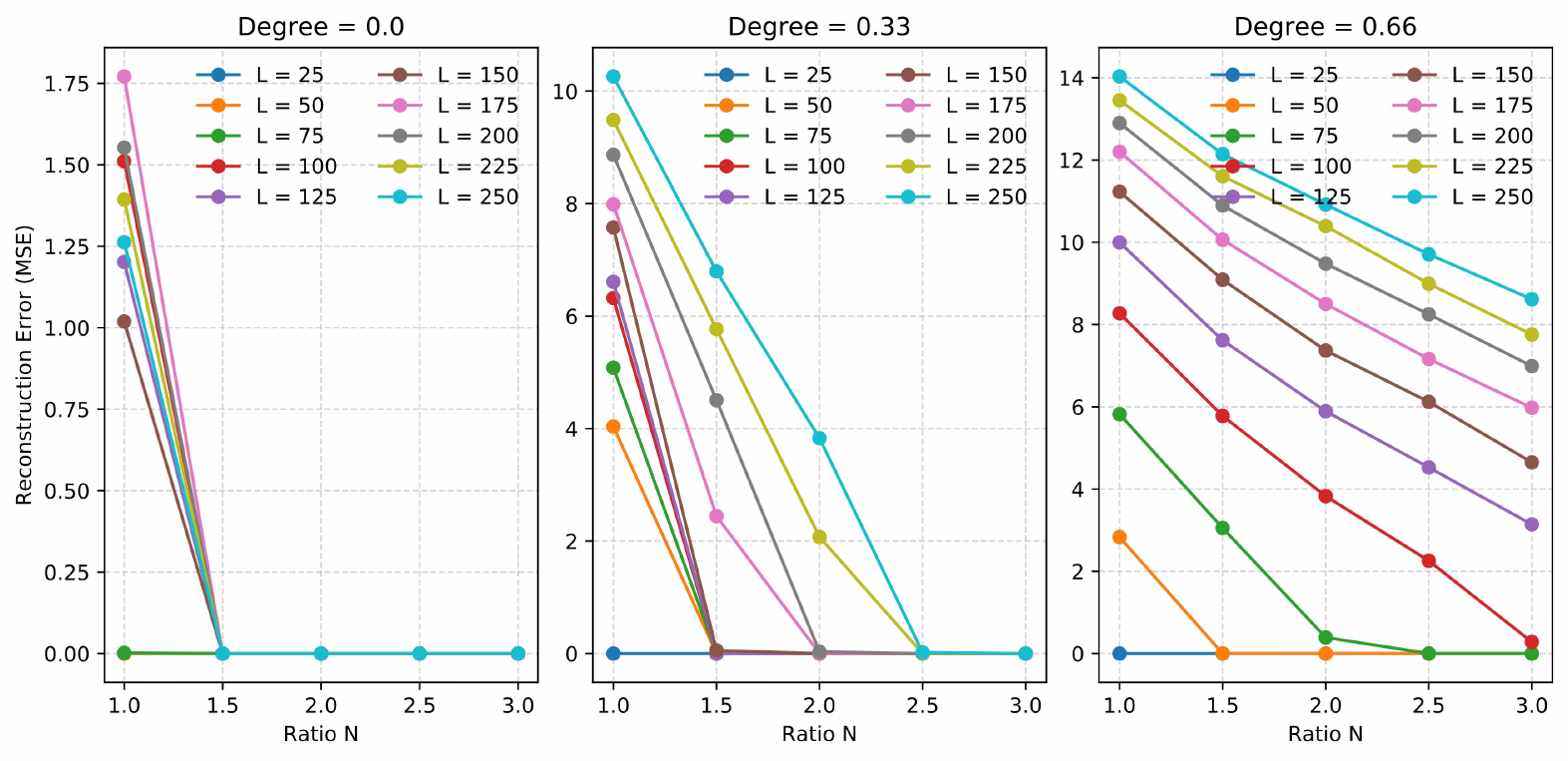}
    \caption{Reconstruction error results of sentence curve inverse mapping/mapping: $E \xrightarrow{B^+} P \xrightarrow{B} E$ process, depending on different sentence length $L$, $N_{ratio}$, and $\eta_{ratio}$.
    }
    \label{fig:reconstruction_analysis}
\end{figure}

To assess the risk of information loss in the $B$ and $B^{+}$ mapping process, we conduct a reconstruction error analysis. Given a length-$L$ random noise sequence sampled from a standard Gaussian distribution, we first map the sequence to the $P$-space using $B^{+}$ and then reconstruct it back to the original space using $B$. We compute the average mean squared error (MSE) between the original and reconstructed sequences. We vary the hyperparameters $L$, $N_{ratio}$, and $\eta_{ratio}$ over the ranges $\{25,50,\ldots,225,250\}$, $\{1.0,1.5,2.0,2.5,3.0\}$, and $\{0.0,0.33,0.66\}$, respectively. The minimum curve degree $\eta$ is fixed to 2, ensuring that at least two control points contribute to each word even when $\eta_{ratio}=0.0$.

Figure~\ref{fig:reconstruction_analysis} reports results averaged over 100 random noise sequences. We observe that shorter sentence lengths yield lower reconstruction errors than longer ones. In addition, increasing the number of control points reduces reconstruction error, whereas higher curve degrees $\eta$ lead to larger errors. Overall, these results suggest that choosing $N_{ratio} \geq 2.0$ and $\eta_{ratio} < 0.33$ provides a safe hyperparameter regime for preserving information in short- to medium-length sentences.

\subsection{Sentence Curve Generation Examples}
\label{appendix:sentence_curve_generation_examples}

\begin{figure}[]
    \centering
    \includegraphics[width=1.0\linewidth]{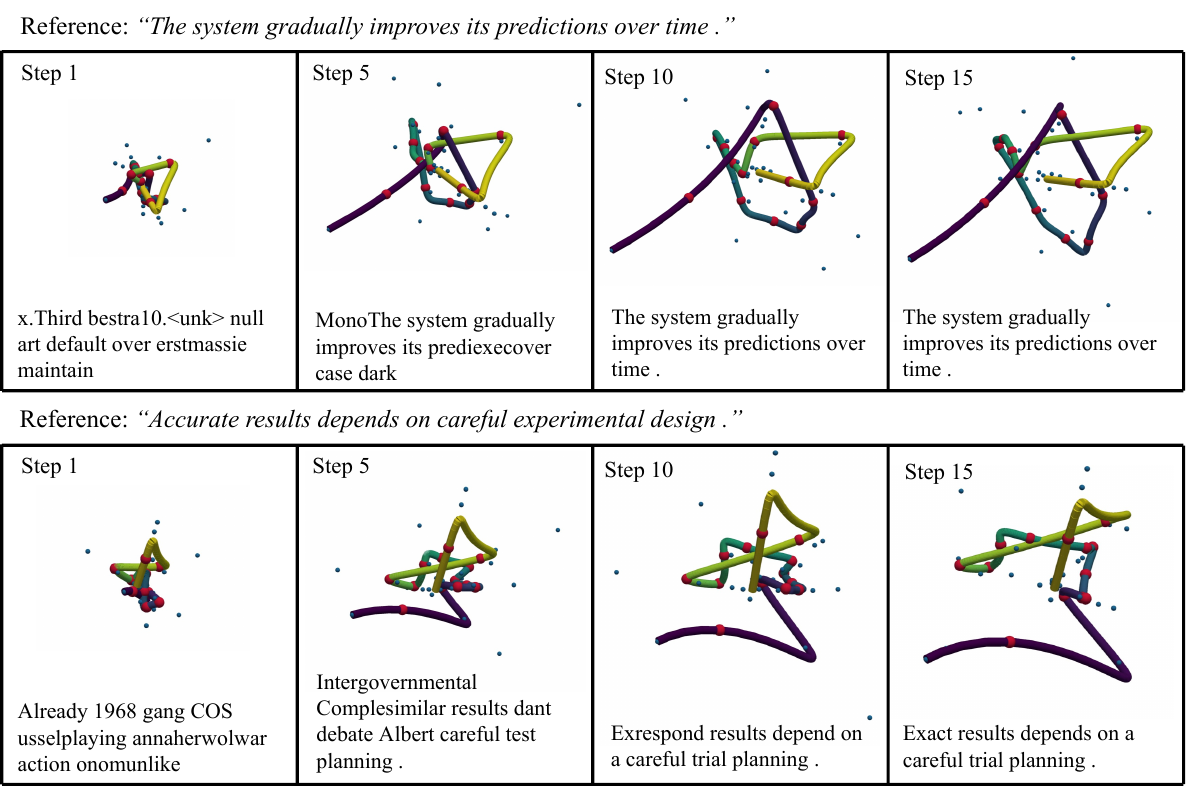}
    \caption{Extra examples of sentence curve generation by SCLM in the translation task. Blue dots represent the estimated control points. Red points on the curve represent the final word embeddings.
    }
    \label{fig:extra_curve_generation_examples}
\end{figure}

In this section, we present additional sentence curve generation examples, analogous to the illustration in the right side of Figure~\ref{fig:front_page}. All those examples are generated using our best-performing WMT14 De$\rightarrow$En checkpoint. Figure~\ref{fig:extra_curve_generation_examples} shows two extra representative sentence curve generation histories. As illustrated, SCLM smoothly denoises the sentence curve from early diffusion steps and progressively refines it toward a curve that maps to the reference sentence embeddings. Together with the examples in Figure~\ref{fig:front_page} and Figure~\ref{fig:extra_curve_generation_examples}, all visualizations are based on the intermediate denoising variables $\hat{\mathbf{e}}_{Y,\theta}^t$.



\end{document}